%% file: neurips_2026.tex
\documentclass{article}
\usepackage[main,final]{neurips_2026}
\setcitestyle{numbers}

\usepackage[utf8]{inputenc}
\usepackage[T1]{fontenc}
\usepackage{amsmath,amsfonts,amssymb,amsthm,mathtools,bm}
\usepackage{nicefrac,microtype}
\usepackage[table]{xcolor}
\definecolor{fluyellow}{RGB}{255,245,80}
\usepackage{graphicx,wrapfig}
\usepackage{booktabs,array,multirow}
\usepackage{enumitem,mdframed,tcolorbox}
\usepackage{algorithm,algpseudocode}
\usepackage{hyperref}
\usepackage{url}
\usepackage[capitalize,noabbrev]{cleveref}

\newtheorem{theorem}{Theorem}
\newtheorem{corollary}{Corollary}

\newtheorem{assumption}{Assumption}
\newtheorem{definition}{Definition}

\title{Combating Catastrophic Forgetting in Continual Domain Adaptation via Knowledge Recasting}

\author{%
  Zijie Liu \\
  UNC at Chapel Hill \\
  jesseliu@cs.unc.edu \\
  \And
  Zhen Tan \\
  Stevens Institute of Technology \\
  ztan12@stevens.edu \\
  \And
  Charles Fleming \\
  Cisco Research \\
  chflemin@cisco.com\\
  \And
  Tianlong Chen \\
  UNC at Chapel Hill \\
  tianlong@cs.unc.edu
}

\begin{document}

\maketitle

\begin{abstract}
Adapting language models to new domains during post-training risks degrading previously acquired capabilities, a phenomenon commonly known as \emph{catastrophic forgetting}. Existing mitigation methods largely treat forgetting as a consequence of excessive model drift. They therefore constrain updates by preserving old outputs, replaying old data, or regularizing changes to parameters and representations so that the updated model remains close to its previous state. We take a \textit{different view}: {Rather than pulling the updated model back toward its old state, we translate the new domain into the model’s pre-existing knowledge space, making unfamiliar knowledge learnable through familiar procedures}. Therefore, we propose \textbf{Skill Schema Transport (SST)}, a post-training framework that transports new knowledge into the model's existing skill space. SST abstracts each example type into a skill schema, preserving its essential operations, dependencies, and verification steps while removing domain-specific surface forms.  By recasting new examples as instances of reusable procedures, SST turns knowledge update from trajectory memorization into procedural reuse. In continual domain adaptation settings, SST consistently reduces forgetting while improving generalization to related tasks. Across two 3B instruction-tuned backbones, SST improves average final performance on streamed domains by 2.53-5.75 points over sequential post-training. It also lowers max-drop forgetting in all evaluated settings, with the largest reduction from 42.41 to 7.10 points on the Qwen-3B MATH stream, while substantially recovering the OOD degradation induced by sequential updates.
Code is available at \url{https://github.com/jesseliu0913/Skill_Forgetting}.
\end{abstract}

\section{Introduction}
\label{sec:intro}
\input{intro}

\section{Related Work}
\label{sec:related_work}
\input{related_work}

\section{Preliminary}
\label{sec:preliminary}
\input{problem_define}

\input{ours}

\section{Experiments}
\label{sec:experiments}
\input{exp}

\vspace{-4mm}
\section{Conclusion}
\vspace{-4mm}
\label{sec:conclusion}

We proposed \emph{Skill Schema Transport} (SST), a post-training framework that mitigates forgetting by expressing new-domain examples through reusable procedural schemas rather than constraining model drift. Across streaming math and agent-task settings, SST improves retention and related-domain transfer over sequential fine-tuning and standard forgetting-mitigation baselines. Further analyses reduce the likelihood that the gains are explained solely by underfitting or a favorable stream order, and checkpoint trajectories are consistent with the view that later domains can reuse and sometimes reinforce shared procedures. These results suggest that continual adaptation can be improved not only by limiting model change, but by representing new knowledge in procedural coordinates the model already understands.

\begin{ack}
This research was partially supported by the National Science Foundation (NSF) under Award No. IIS-2551752 and a Cisco Faculty Award.
\end{ack}

\bibliographystyle{plain}
\bibliography{references}

\clearpage
\appendix
\input{appendix}

\end{document}

%% file: intro.tex
\begin{figure}[t]
    \centering
    \vspace{-3mm}
    \includegraphics[width=1\textwidth]{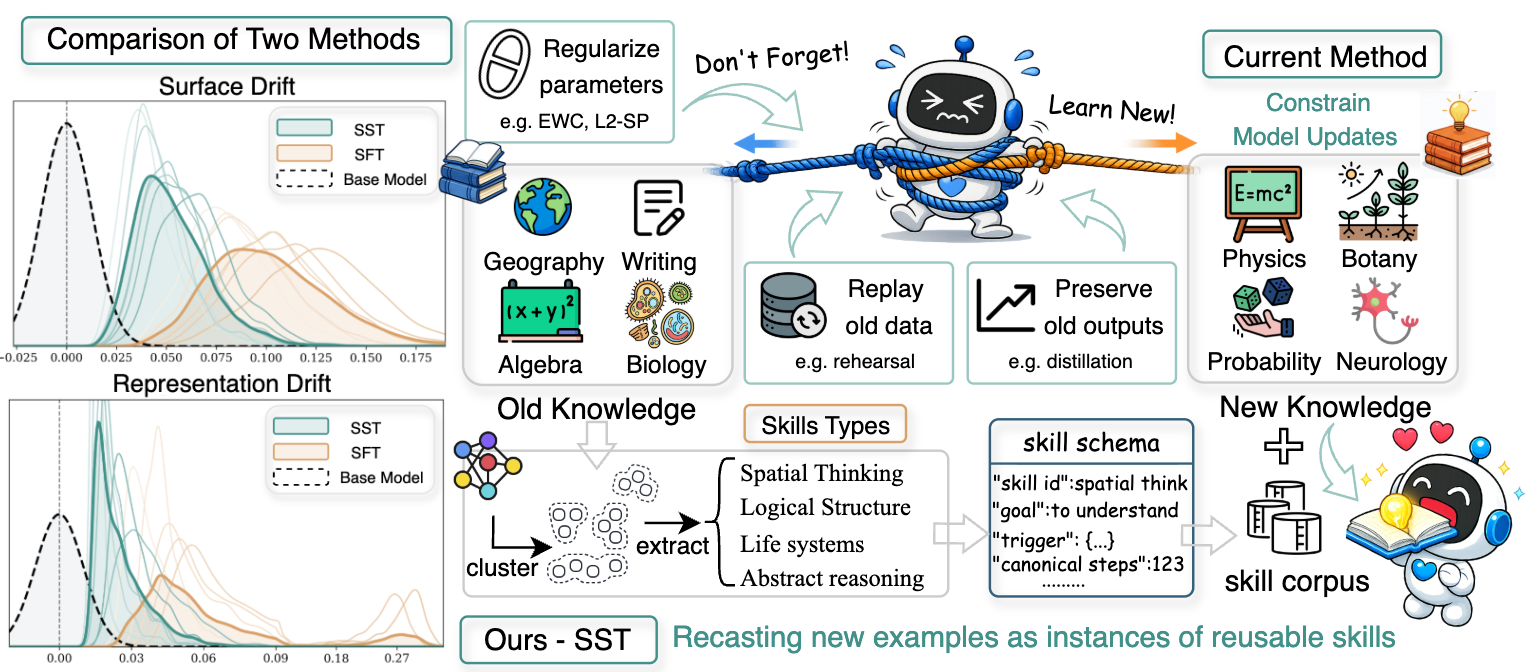}
    \vspace{-3mm}
    \caption{\small
    \textbf{Overview of Skill Schema Transport (SST).}
    SST recasts new-domain examples as instances of reusable skill schemas rather than anchoring old data, output distributions, or parameter values.
    The left panels illustrate the drift diagnostics used in our analysis, and the bottom pipeline shows how procedural skills are extracted, abstracted, stored, and used for schema-conditioned updates.
    }
    \label{fig:tease_image}
    \vspace{-3mm}
\end{figure}

Post-training improves language models (LMs) through supervised fine-tuning~\citep{howard2018universal}, reinforcement learning from feedback~\citep{ouyang2022training,bai2022constitutional}, and domain-adaptive continued training~\citep{gururangan2020don}.
However, sequential post-training can also degrade capabilities acquired before the update, a phenomenon commonly known as \emph{catastrophic forgetting}~\citep{mccloskey1989catastrophic,kirkpatrick2017overcoming,luo2023empirical,shi2025continual,wu2024continual}.
This risk becomes especially important under domain shift, where the update distribution introduces concepts, terminology, or reasoning contexts that differ from those emphasized by the pre-update model~\citep{shi2025continual,huang2024mitigating,zheng2025spurious}.

To clarify what existing mitigation strategies are designed to control, we view post-training under domain shift through the lens of model drift.
As shown in Figure~\ref{fig:two_drifts}, knowledge updates can change a model in at least two distinct ways.
\emph{Surface drift} refers to changes in the model's observable behavior, such as its output distribution or final answers, whereas \emph{representation drift} refers to changes in the internal states that support those behaviors.
These two forms of drift need not agree: the mapping from representations to outputs is many-to-one, so a model may reorganize its internal states while producing similar outputs, or substantially shift its output distribution while leaving representations relatively stable.

\begin{wrapfigure}{l}{0.35\textwidth}
    \centering
    \vspace{-3mm}
    \includegraphics[width=0.34\textwidth]{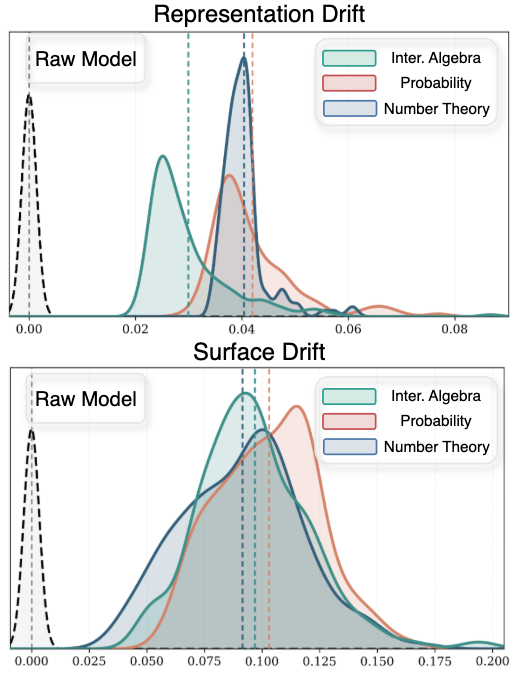}
    \caption{\small
    \textbf{Surface drift and representation drift need not align.}
    Domain-shifted updates induce different distributions of output-level and hidden-state change, motivating diagnostics that track both observable behavior and internal representations.
    }
    \label{fig:two_drifts}
    \vspace{-5mm}
\end{wrapfigure}

From this perspective, many forgetting-mitigation methods can be understood as different ways of anchoring the updated model to its pre-update behavior or state.
Replay-based methods mix old examples with new training data so that the model continues to reproduce earlier task behavior during the update~\citep{rolnick2019experience,lopez2017gradient,huang2024mitigating}.
Distillation-based methods use the original model as a teacher and penalize deviations from its outputs~\citep{li2016learning}.
Both families primarily constrain surface drift by preserving observable behavior on old examples.
In contrast, regularization-based methods penalize changes to important parameters or restrict updates to protected subspaces~\citep{kirkpatrick2017overcoming,zenke2017continual,aljundi2018memory,wang2023orthogonal,lu2025controlled}, thereby limiting drift in the model's internal state or parameterization.

Because these approaches primarily anchor one aspect of the update, anchoring surface behavior alone need not control representation drift, and anchoring parameter or representation movement alone need not preserve task behavior.
This points to a limitation of treating forgetting mitigation only as a problem of preventing change: the key question is not merely how much the model changes, but which aspects of the update are being constrained and whether those constraints help integrate the new domain.

Yet anchoring the model is not the same as integrating new knowledge.
Under out-of-domain updates, useful adaptation requires change: the model may need to acquire, emphasize, or recombine concepts and reasoning contexts that were weakly represented in the pre-update distribution.
\emph{The key issue is therefore not how to minimize change, but how to organize it.}
The organizing principle we propose is \textbf{procedural reuse}.
If a new problem can be solved by a procedure the model already uses, then the desirable update is not to memorize a new trajectory or to pull the model back toward its old outputs.
Instead, the update should connect the unfamiliar instance to a familiar procedure, expressing new knowledge in procedural coordinates the model can reuse.
Inspired by this view, we ask the following research question:

\begin{mdframed}[nobreak=true,userdefinedwidth=.99\linewidth,align=center,skipabove=3pt,skipbelow=3pt,innerleftmargin=6pt,innerbottommargin=6pt,innertopmargin=6pt,roundcorner=3pt,backgroundcolor=cyan!5,linecolor=gray]
\noindent \textbf{RQ}: \textit{Can new-domain knowledge be expressed in the procedural language of the model's existing skills, so that unfamiliar problems become learnable through familiar procedures?}
\end{mdframed}

To probe whether this premise is plausible, we analyze the error structure of a base model across MATH categories.
The same failed examples are embedded and projected in Figure~\ref{fig:error_clustering}; coloring them by procedural error type reveals clusters that cut across the original domain labels.
This diagnostic does not by itself prove that all domains share a complete skill library, but it suggests that observed failures are not purely domain-specific and motivates testing whether schema-level reuse can improve continual adaptation.

\begin{wrapfigure}{r}{0.65\textwidth}
    \centering
    \vspace{-6mm}\includegraphics[width=0.64\textwidth]{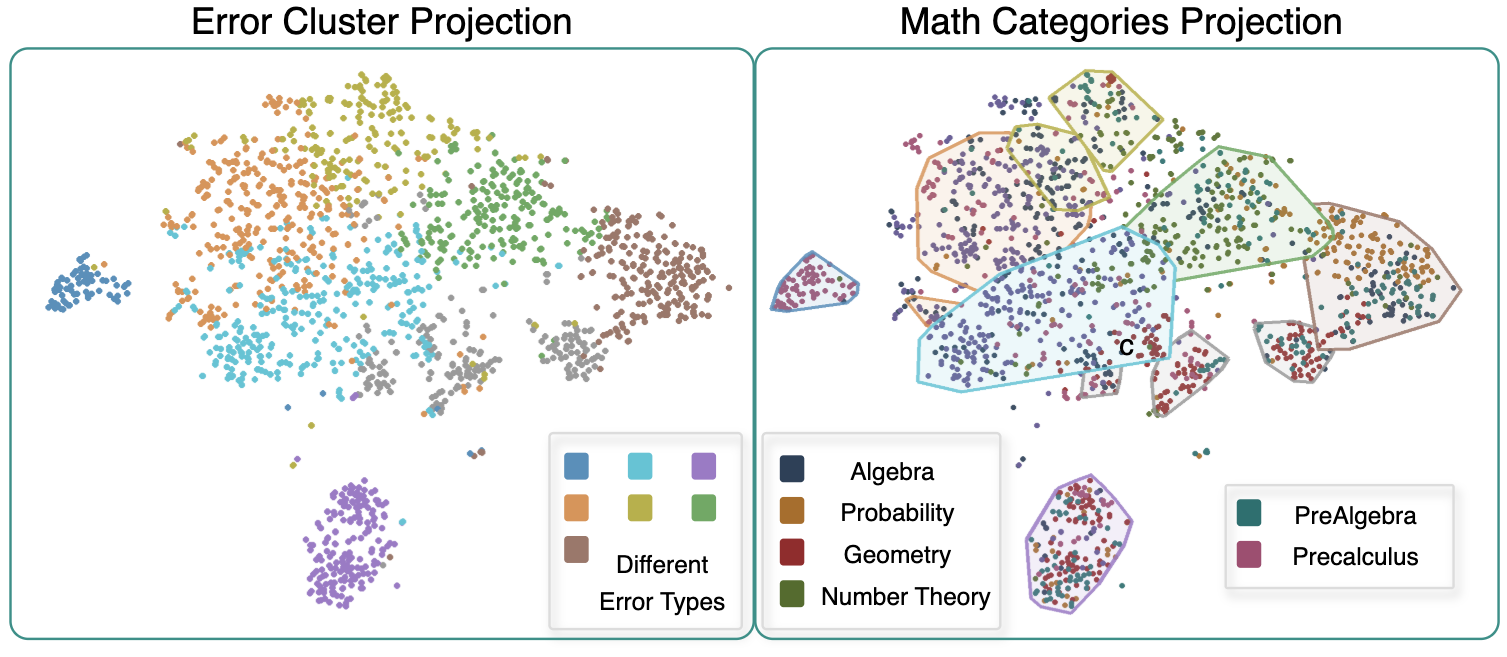}
    \caption{\small
    \textbf{Errors cluster by procedure rather than by domain label.}
    The same failed examples are projected and colored by procedural error type on the left and by MATH category on the right.
    Error clusters cut across category boundaries, motivating cross-domain skill schema reuse.
    }
    \label{fig:error_clustering}
    \vspace{-3mm}
\end{wrapfigure}

Based on this observation, we propose \emph{Skill Schema Transport} (SST), a schema-conditioned post-training framework that transports new-domain examples into a reusable procedural schema space.
A skill schema serves as the medium of this translation: it retains the operations, dependencies, and checks needed to solve a task, while abstracting away domain-specific wording, formats, and trajectories.
SST uses these schemas to associate unfamiliar examples with familiar procedures, allowing a new-domain example to be interpreted as a transported instance of a reusable skill rather than as an isolated new trajectory, as illustrated in Figure~\ref{fig:tease_image}.
We evaluate SST in streaming knowledge-update settings, where new-domain knowledge is introduced sequentially and the model must retain capabilities learned from earlier domains.
Across two 3B instruction-tuned backbones, a four-domain MATH stream, and a three-domain agent stream, SST improves average final performance on streamed domains by 2.53-5.75 points over the corresponding sequential post-training baseline.
It also reduces max-drop forgetting in all evaluated settings and maintains or substantially recovers OOD benchmark performance relative to sequential updates. To summarize, our main contributions are:
\begin{itemize}[leftmargin=*]
    \item We reframe forgetting mitigation under domain shift as a problem of knowledge organization rather than only model constraint, shifting the burden of adaptation from limiting model movement to structuring how new examples are represented during training.

    \item We provide drift and error-structure diagnostics showing that surface and representation drift need not align, and that procedural failure patterns can cut across domain boundaries, motivating procedural reuse as an alternative to purely drift-based anchoring.

    \item We propose \emph{Skill Schema Transport} (SST), a schema-conditioned post-training framework that extracts domain-agnostic skill schemas and transports new-domain examples into reusable procedural coordinates.

    \item Across two streaming settings comprising seven update domains and two 3B instruction-tuned backbones, SST improves streamed-domain performance by 2.53-5.75 points over sequential post-training, reduces max-drop forgetting in all evaluated settings, and maintains or substantially recovers OOD benchmark performance relative to sequential updates.
\end{itemize}

%% file: related_work.tex
\textbf{Continual Learning for Sequential LLM Training.}
Continual learning studies how models acquire new capabilities without overwriting old ones, with catastrophic forgetting as a central failure mode \citep{mccloskey1989catastrophic,robins1995catastrophic,french1999catastrophic,goodfellow2013empirical}. Existing methods mitigate forgetting through parameter consolidation, functional regularization, replay, and gradient constraints, including EWC, SI, MAS, and related approaches \citep{kirkpatrick2017overcoming,zenke2017continual,aljundi2018memory,chaudhry2018riemannian,schwarz2018progress,li2016learning,rebuffi2017icarlincrementalclassifierrepresentation,lopez2017gradient,chaudhry2018efficient,buzzega2020dark,parisi2019continual,de2021continual}. Analogous ideas appear in sequential RL through replay, distillation, and policy regularization, while recent LLM post-training studies show that continual pretraining, instruction tuning, SFT, and RLHF can all induce forgetting under domain shift \citep{rolnick2019experience,kaplanis2019policy,ouyang2022training,bai2022constitutional,luo2023empirical,wu2024continual,shi2025continual,zheng2025spurious}. Parameter-efficient methods such as LoRA, O-LoRA, and CLoRA reduce interference by restricting adaptation to low-rank, orthogonal, or regularized subspaces, but they still mainly act as drift-control mechanisms \citep{hu2022lora,wang2023orthogonal,lu2025controlled}. Our work is complementary: instead of only constraining surface or representation drift, we ask whether new-domain knowledge can be transported into procedures already learned by the model, and analyze this using representation-similarity tools such as SVCCA, PWCCA, and CKA~\citep{raghu2017svcca,morcos2018insights,kornblith2019similarity}.

\textbf{Skills, Temporal Abstraction, and Skill Discovery.}
The idea that complex behavior can be decomposed into reusable subroutines is central to hierarchical reinforcement learning, where Options, MAXQ, feudal RL, and Option-Critic define skills as temporally extended actions or subtasks that support behavioral reuse \citep{sutton1999between,dietterich2000hierarchical,dayan1992feudal,barto2003recent,bacon2017option,pateria2021hierarchical}. A related line studies unsupervised skill discovery, where methods such as VIC, DIAYN, DADS, and OPAL learn diverse, predictable, or recurring primitives without task-specific supervision \citep{gregor2016variational,eysenbach2018diversity,sharma2019dynamics,ajay2021opal}. We build on this abstraction view, but do not treat skills as options, policies, or transfer modules. Instead, we define skills as language-level schemas that preserve operations, dependencies, and checks while abstracting away domain-specific wording, formats, and trajectories, allowing us to test whether schema-based supervision induces less unnecessary representation drift during post-training \citep{raghu2017svcca,kornblith2019similarity}.

\textbf{Skill-Augmented and Continually Improving Agents.}
Recent LLM-agent systems externalize reusable experience as memories, reflections, trajectories, executable programs, or skill libraries. ReAct interleaves reasoning with environment actions, while Reflexion and ExpeL improve agents through verbal feedback, episodic memory, or distilled experience without updating model weights \citep{yao2022react,shinn2023reflexion,zhao2023expel}. Voyager extends this idea to open-ended lifelong learning with an executable skill library for embodied control \citep{wang2023voyager}. More recent work makes skill learning explicit through recursive skill libraries, persistent skill memories, multimodal experience-skill stores, and skill-oriented benchmarks~\citep{xia2026skillrl,li2026skillsbench,jiang2026xskill,zhou2026memento}. These methods typically use skills as external memories, prompts, tools, or inference-time resources, whereas our work studies schema-conditioned sequential post-training with retrieval-free main evaluation: SST uses skill schemas during training to encourage new-domain knowledge to be internalized through reusable procedures.

%% file: problem_define.tex
We formalize the continual domain adaptation setting used throughout the paper. We first define the update and retention distributions, then define forgetting and drift diagnostics, and finally introduce the schema objects used by SST.

\begin{definition}[Continual Learning on Shifting Domains]
\label{def:continual_domain_shifting}
Let $\mathcal{M}_{\theta_0}$ be a pre-trained language model and $\mathcal{D}_1,\ldots,\mathcal{D}_T$ a sequence of update distributions, where each $\mathcal{D}_t$ is supported on a domain $\Omega_t$ (e.g., a subject area such as algebra, medicine, or law) that may differ from the domains seen during pre-training. At step $t$, the model is updated from $\theta_{t-1}$ to $\theta_t$ using examples drawn from $\mathcal{D}_t$, while retaining performance on a set of retention distributions $\mathfrak{R}_t \subseteq \{\mathcal{Q}_i:1\le i<t\}$, where each $\mathcal{Q}_i$ is supported on the previously encountered domain $\Omega_i$.
\end{definition}

The setting specifies where new learning occurs and where retention is
measured. We next define forgetting as the maximum accuracy drop over the
observed stream.
\begin{definition}[Forgetting Gap]
\label{def:forgetting_gap}
For a retention distribution $\mathcal{Q}\in\mathfrak{R}_t$, let
$A_t(\mathcal{Q})$ denote the task accuracy on $\mathcal{Q}$ after update step
$t$. If $\mathcal{Q}$ corresponds to a domain introduced at step $t_{\mathcal{Q}}$,
we define its forgetting gap after the full stream as
\begin{equation}
    \mathcal{F}_{\mathcal{Q}}
    =
    \max_{t_{\mathcal{Q}}\le t\le T} A_t(\mathcal{Q})
    -
    A_T(\mathcal{Q}).
\end{equation}
We report average forgetting by averaging $\mathcal{F}_{\mathcal{Q}}$ over prior
retention distributions, excluding the final streamed domain because it is not
followed by a later update and therefore cannot be overwritten within the
observed stream.
\end{definition}

The forgetting gap measures the outcome of the update stream. To diagnose the updates themselves, we also track how much the model changes at the output and representation levels.

\begin{definition}[Drift Diagnostics]
\label{def:drift}
For an input $x$, let $p_\theta(\cdot\mid x)$ denote the model's output distribution over the vocabulary and $h_\theta(x)\in\mathbb{R}^d$ a hidden representation extracted from a fixed intermediate layer. Let $\mathcal{Q}_X$ denote the marginal distribution over inputs $x$. The cumulative drift from $\theta_0$ and the incremental drift from $\bar\theta=\theta_{t-1}$ are defined, respectively, as
\begin{align}
    G_{\mathrm{surface}}^{0}(\theta;\mathcal{Q})
    &= \mathbb{E}_{x\sim\mathcal{Q}_X}
    \left[\mathrm{JS}\!\left(p_{\theta}(\cdot|x)\,\Vert\,p_{\theta_0}(\cdot|x)\right)\right],
    \\
    G_{\mathrm{surface},\bar\theta}(\theta;\mathcal{Q})
    &= \mathbb{E}_{x\sim\mathcal{Q}_X}
    \left[\mathrm{JS}\!\left(p_{\theta}(\cdot|x)\,\Vert\,p_{\bar\theta}(\cdot|x)\right)\right],\\
    G_{\mathrm{repr}}^{0}(\theta;\mathcal{Q})
    &= \mathbb{E}_{x\sim\mathcal{Q}_X}
    \left[\mathrm{Dist}\!\left(h_{\theta}(x),\,h_{\theta_0}(x)\right)\right],
    \\
    G_{\mathrm{repr},\bar\theta}(\theta;\mathcal{Q})
    &= \mathbb{E}_{x\sim\mathcal{Q}_X}
    \left[\mathrm{Dist}\!\left(h_{\theta}(x),\,h_{\bar\theta}(x)\right)\right],
\end{align}

\end{definition}

These drift quantities measure changes in the model, but they do not specify
whether two examples use the same solution procedure. SST therefore uses a
schema-level representation of procedural structure.
We represent procedural structure via skill schemas. Let $\tau\in\mathcal{T}$
denote a reasoning trace for solving $x$, and let $\mathcal{B}_t$ denote the
schema memory accumulated through step $t$.
\input{figures/case}
\begin{definition}[Skill Schema]
\label{def:skill_schema}
A skill schema is an abstraction $z = \Gamma_\theta(\tau) \in \mathcal{Z}$,
where $\Gamma_\theta: \mathcal{T} \to \mathcal{Z}$ maps a reasoning trace $\tau$ to a
structured, domain-agnostic representation of its procedural steps.
\end{definition}
Table~\ref{tab:skill-schema-example} shows one instantiated schema and its
fields. It specifies what information is stored in a schema, including its
goal, applicability conditions, canonical steps, failure modes, and
compositional links. Figure~\ref{fig:skill_sequence_example} illustrates how a concrete problem is paired with a transported schema to form a schema-conditioned SFT example.

\begin{definition}[Schema Transportability]
\label{def:transportability}
A new example with schema $z = \Gamma_{\theta_{t-1}}(\tau)$ is
$\epsilon$-\emph{schema-transportable} with respect to a nonempty schema memory
$\mathcal{B}$ if
\begin{equation}
    \mathrm{Dist}\!\left(z,\,\mathrm{NN}_{\mathcal{B}}(z)\right)\leq\epsilon,
    \qquad
    \mathrm{NN}_{\mathcal{B}}(z) \in \arg\min_{z'\in\mathcal{B}}
    \mathrm{Dist}(z,\,z'),
\end{equation}
where $\mathrm{NN}_{\mathcal{B}}(z)$ denotes the nearest neighbor of $z$ in
the schema memory $\mathcal{B}$ under $\mathrm{Dist}$, with ties broken
deterministically. If $\mathcal{B}=\varnothing$, the example is treated as
non-transportable.
\end{definition}

Transportability determines how new examples interact with the schema memory
$\mathcal{B}_t$: transportable examples are attached to an existing procedure,
while non-transportable examples introduce a new one. The maintenance operations
on $\mathcal{B}_t$ are detailed in Section~\ref{sec:method}.

%% file: figures/case.tex
\begin{figure}[t]
\centering
\small
\setlength{\fboxsep}{6pt}
\fbox{\begin{minipage}{0.95\linewidth}
\textbf{Input:} problem $q$, transported skill schema $\widetilde z$, target response $y$. \vspace{1mm}

\textbf{[Problem $q$]} Solve for $u$: $\displaystyle \frac{u}{2}+3=7$. \vspace{1mm}

\textbf{[Transported skill schema $\widetilde z$]}
\texttt{Algebraic Substitution and Equation Solving}.
\textit{Canonical steps:} rewrite the relation; isolate the variable; solve; verify the result. \vspace{1mm}

\textbf{[Augmented training prompt]}
\begin{quote}\vspace{-1mm}
Problem: Solve for $u$: $\displaystyle \frac{u}{2}+3=7$.
Relevant skill: Algebraic Substitution and Equation Solving.
Use the following procedure: rewrite the relation, isolate the variable, solve, and verify the result.
\vspace{-1mm}\end{quote}

\textbf{[Target response $y$]}
\begin{quote}\vspace{-1mm}
Subtract $3$ from both sides to get $\displaystyle \frac{u}{2}=4$.
Multiply both sides by $2$ to get $u=8$.
Check: $\displaystyle \frac{8}{2}+3=4+3=7$, so the solution is correct.
\vspace{-1mm}\end{quote}
\end{minipage}}
\caption{\small
A schema-conditioned SFT training example in SST.
The setup is identical to standard SFT except that the input prompt is augmented with a transported skill schema $\widetilde z$; the supervision target remains the gold response $y$ only.
}
\label{fig:skill_sequence_example}
\vspace{-3mm}
\end{figure}

%% file: ours.tex
\section{Skill Schema Transport}
\label{sec:method}

We propose \emph{Skill Schema Transport} (SST), a post-training framework that connects new-domain knowledge to procedures learned from earlier domains, shown in Algorithm~\ref{alg:sst}. Rather than constraining the model to reproduce previous outputs, SST uses skill schemas as procedural descriptors, allowing a new-domain example to be interpreted through a reusable skill whenever its schema matches one stored in memory.

\textbf{Schema Extraction.} For each training example $(x, y, \tau) \sim
\mathcal{D}_t$, the reasoning trace $\tau$ is the ground-truth solution
provided by the dataset. The skill schema $z = \Gamma_{\theta_{t-1}}(\tau)$
is extracted using the model prior to the current update step via a structured
prompt (detailed in Appendix~\ref{app:prompts}). The extractor
$\Gamma_{\theta_{t-1}}$ takes $\tau$ as input and produces a domain-agnostic
schema $z \in \mathcal{Z}$ specifying the procedural structure of the
solution, including operations, dependencies, and validation checks, while
abstracting away domain-specific terminology and surface details.
\input{figures/skill_schema}
\paragraph{Schema Memory Initialization and Skill Retrieval.}
The schema memory $\mathcal{B}_0$ is initialized from a fixed set of predefined skill families and seed skills before any examples from $\mathcal{D}_1$ are processed.
During update step $t$, SST maintains a working memory $\bar{\mathcal B}$ initialized as $\mathcal{B}_{t-1}$.
For each example in $\mathcal{D}_t$, the system extracts a candidate schema using $\Gamma_{\theta_{t-1}}$, retrieves the nearest skill from $\bar{\mathcal B}$, constructs the transported schema $\widetilde z$, and immediately applies a memory-update operation to $\bar{\mathcal B}$.
After all examples from $\mathcal{D}_t$ are processed, the working memory is committed as $\mathcal{B}_t$.
New skills are written into memory only if the same procedural signature receives at least $n_{\min}=2$ supporting examples, preventing spurious entries from single occurrences.
Schema transportability (Definition~5) is defined conceptually over the schema space $\mathcal Z$.
In implementation, we operationalize this with a problem-to-skill retrieval proxy: the input problem $x$ and each skill entry in the current working memory $\bar{\mathcal B}$ are encoded using \texttt{bge-small-en-v1.5}~\citep{bge_embedding} and ranked by cosine similarity.
Each skill is represented as a canonicalized natural-language block comprising five fields: Skill ID, Family, Goal, When to use, Steps, and Common failure modes.
The top-$k$ retrieved skill blocks ($k=2$ in our experiments) are then used to condition schema extraction and the construction of the transported schema $\widetilde z$.
\input{figures/main_table}

\textbf{Schema-Conditioned Post-Training.} At update step $t$, standard SFT
optimizes
\begin{equation}
    \mathcal{L}_{\mathrm{SFT}}^t(\theta)
    =
    \mathbb{E}_{(x,y)\sim\mathcal{D}_t}
    \left[-\log p_\theta(y\mid x)\right].
\end{equation}
When $\mathcal{D}_t$ is out-of-domain, this objective may cause the model to
acquire isolated domain-specific procedures rather than expressing new
knowledge in terms of procedures already present in its repertoire. SST
processes examples sequentially with a working memory
$\bar{\mathcal{B}}_{t,0}=\mathcal{B}_{t-1}$. For the $j$-th example
$(x_j,y_j,\tau_j)$, let $z_j=\Gamma_{\theta_{t-1}}(\tau_j)$ and assign
\begin{equation}
\label{eq:transported_schema}
    \widetilde{z}_j
    =
    \begin{cases}
    \mathrm{NN}_{\bar{\mathcal{B}}_{t,j-1}}(z_j),
    & \text{if } \bar{\mathcal{B}}_{t,j-1}\neq\emptyset \text{ and }
      \mathrm{Dist}\!\left(
        z_j,\mathrm{NN}_{\bar{\mathcal{B}}_{t,j-1}}(z_j)
      \right)
      \leq\epsilon,\\[6pt]
    z_j, & \text{otherwise}.
    \end{cases}
\end{equation}
Here $\bar{\mathcal{B}}_{t,j-1}$ is the current working memory before processing
example $j$, and $\mathrm{NN}_{\bar{\mathcal{B}}_{t,j-1}}(z_j)$ denotes the
nearest neighbor of $z_j$ in that memory under $\mathrm{Dist}$. This yields the
schema-conditioned dataset
\[
    \mathcal{D}_t^{\mathrm{SST}}
    =
    \{(x_j,\widetilde{z}_j,y_j)\}_{j=1}^{|\mathcal{D}_t|}
\]
and objective
\begin{equation}
\label{eq:sst_objective}
    \mathcal{L}_{\mathrm{SST}}^t(\theta)
    =
    \mathbb{E}_{(x,\widetilde{z},y)\sim\mathcal{D}_t^{\mathrm{SST}}}
    \left[-\log p_\theta(y\mid x,\widetilde{z})\right].
\end{equation}
Here $\widetilde{z}$ serves solely as input context and is not a supervision
target: the model is trained to generate $y$ conditioned on both the problem
$x$ and the retrieved procedural context $\widetilde{z}$, but is not
explicitly trained to predict $\widetilde{z}$ itself. Thus $x$ provides
instance-specific content, $y$ provides the target response, and $\widetilde{z}$
provides the procedural coordinates through which the new example is expressed.
\input{algo/sst}

\paragraph{Schema Memory Maintenance.}
The schema memory $\mathcal{B}_t$ evolves across update steps without requiring replay access to $\mathcal{D}_{<t}$.
Memory updates occur per example, not as a separate post-hoc phase.
During stage $t$, all schema extraction and memory-operation proposals are produced by the pre-update model $\Gamma_{\theta_{t-1}}$; the updated model $\theta_t$ is obtained only after the schema-conditioned dataset $\widetilde{\mathcal D}_t$ has been constructed.
After $\widetilde z$ is assigned for a given example, $\Gamma_{\theta_{t-1}}$ produces an update proposal and applies either a Reuse no-op or one of four memory-modifying operations.
\emph{Revise:} if the proposed schema matches an existing entry within the
same skill family, the stored schema is updated with new keywords, canonical
steps, or failure modes.
\emph{Create:} if $\bar{\mathcal B}=\varnothing$ or
$\mathrm{Dist}(z,\mathrm{NN}_{\bar{\mathcal B}}(z))>\epsilon$, the proposed
schema is a candidate for a new entry; it is written
into $\mathcal{B}_t$ only when the same procedural signature accumulates at
least $n_{\mathrm{min}}$ supporting examples and the extractor's confidence
exceeds a threshold $\delta_{\mathrm{create}}$.
\emph{Merge:} if two schemas share the same skill family and satisfy
$\mathrm{Dist}(z, z') \leq \delta_{\mathrm{merge}}$, they are consolidated,
retaining the more-retrieved entry.
\emph{Drop:} a schema whose utility score $u(z)$ falls below
$\delta_{\mathrm{drop}}$ is removed, provided it has been retrieved at least
$k_{\mathrm{min}}$ times. The utility score aggregates behavioral signals
including source correctness, retrieval support, retrieval precision. Implementation details and default hyperparameters are provided in Appendix~\ref{app:implementation_details}.
Skill family membership is a discrete field assigned by the extractor at
creation time from a fixed family list; it governs which skills are candidates
for Revise and Merge, while $\mathrm{Dist}$ governs the within-family
similarity threshold for Merge.
Together, these operations allow $\mathcal{B}_t$ to grow, refine, and prune itself as new domains are encountered.
A local linearized analysis in Appendix~\ref{app:sst_theory_details} shows that, under a schema-separable approximation, schema-conditioned directions need not increase the minimum shared-parameter displacement required by an update.

%% file: figures/skill_schema.tex
\begin{wraptable}{r}{0.56\textwidth}
\vspace{-4mm}
\centering
\scriptsize
\renewcommand{\arraystretch}{1.05}
\setlength{\tabcolsep}{3pt}

\caption{Example of a structured skill schema used.}
\label{tab:skill-schema-example}
\vspace{2mm}
\begin{tabular}{>{\bfseries\raggedright\arraybackslash}p{1.85cm}
                >{\raggedright\arraybackslash}p{5.0cm}}
\toprule
\rowcolor{gray!15}
Field & Description \\
\midrule
Skill Name & Algebraic Substitution and Equation Solving \\

Goal & Solve algebra problems via rewriting, substitution, and branch verification. \\

\rowcolor{gray!7}
Applicability & Equations, expressions, radicals, and symbolic substitutions. \\

Keywords & equation, substitute, solve, factor, expression \\

\rowcolor{gray!7}
Structural Cues & functional relations; inequalities or intervals \\

Output Format & integer, fraction, radical, interval \\

\rowcolor{gray!7}
Canonical Steps &
(1) Rewrite the relation; (2) substitute or factor strategically;
(3) solve candidate values; (4) check constraints. \\

Failure Modes & sign errors; dropped branches; unchecked extraneous solutions \\

\rowcolor{gray!7}
Compositional Link & exact form and answer realization \\
\bottomrule
\end{tabular}
\vspace{-6mm}
\end{wraptable}

%% file: figures/main_table.tex
\begin{table*}[t]
\centering
\scriptsize
\setlength{\tabcolsep}{2.8pt}
\caption{\small
\textbf{Main streaming results.}
All numbers are percentages after the full stream.
\textit{Stream Avg.}: average over streamed domains.
\textit{Forget.}: average max-drop over prior streamed domains.
\textit{OOD Avg.}: average over MMLU, GSM8K, PIQA, CSQA, and SciQ (flex score for GSM8K).
Sequential baseline is Seq-SFT for MATH and Seq-RL for agent tasks.
SST uses the same update objective as its sequential counterpart, conditioned on skill schemas.
Forgetting is not applicable to base models.
}
\label{tab:main_streaming}
\vspace{1mm}

\resizebox{\textwidth}{!}{
\begin{tabular}{llrrrrrr|rrrrrr}
\toprule
\multicolumn{14}{c}{\textbf{MATH stream: Algebra $\to$ Counting \& Probability $\to$ Geometry $\to$ Intermediate Algebra}} \\
\midrule
& & \multicolumn{6}{c|}{\textbf{Streamed domains}}
& \multicolumn{6}{c}{\textbf{OOD generalization}} \\
\cmidrule(lr){3-8}\cmidrule(lr){9-14}
Model & Method
& Alg. & Count. & Geom. & Inter.
& Stream Avg. $\uparrow$ & Forget. $\downarrow$
& MMLU & GSM8K & PIQA & CSQA & SciQ & OOD Avg. $\uparrow$ \\
\midrule
Qwen-3B & Base
& 44.99 & 28.48 & 22.34 & 9.19
& 26.25 & --
& 65.45 & 63.00 & 77.69 & 75.54 & 92.50 & 74.84 \\
Qwen-3B & Seq-SFT
& 26.62 & 21.31 & 18.16 & 13.73
& 19.96 & 42.41
& 57.18 & 7.58 & 74.43 & 61.59 & 93.80 & 58.92 \\
\rowcolor{fluyellow!25} Qwen-3B & SST
& 40.27 & 24.68 & 20.46 & 15.06
& \textbf{25.12} & \textbf{7.10}
& 65.28 & 39.20 & 77.58 & 80.56 & 97.70 & \textbf{72.06} \\
\midrule
Llama-3B & Base
& 8.17 & 5.70 & 11.27 & 9.63
& 8.69 & --
& 62.12 & 67.48 & 73.90 & 73.96 & 94.20 & 74.33 \\
Llama-3B & Seq-SFT
& 30.03 & 20.16 & 18.66 & 12.33
& 20.30 & 4.39
& 60.43 & 46.40 & 73.72 & 71.01 & 95.10 & 69.33 \\
\rowcolor{fluyellow!25} Llama-3B & SST
& 31.76 & 21.73 & 20.88 & 16.94
& \textbf{22.83} & \textbf{4.09}
& 60.58 & 61.84 & 74.32 & 70.60 & 94.30 & \textbf{72.33} \\
\bottomrule
\end{tabular}
}

\vspace{1.5mm}

\resizebox{\textwidth}{!}{
\begin{tabular}{llrrrrr|rrrrrr}
\toprule
\multicolumn{13}{c}{\textbf{Agent stream: AlfWorld $\to$ WebShop $\to$ Search}} \\
\midrule
& & \multicolumn{5}{c|}{\textbf{Streamed domains}}
& \multicolumn{6}{c}{\textbf{OOD generalization}} \\
\cmidrule(lr){3-7}\cmidrule(lr){8-13}
Model & Method
& AlfWorld & WebShop & Search QA
& Stream Avg. $\uparrow$ & Forget. $\downarrow$
& MMLU & GSM8K & PIQA & CSQA & SciQ & OOD Avg. $\uparrow$ \\
\midrule
Qwen-3B & Base
& 19.79 & 42.75 & 53.72
& 38.75 & --
& 65.45 & 63.00 & 77.69 & 75.54 & 92.50 & 74.84 \\
Qwen-3B & Seq-RL
& 66.97 & 91.33 & 60.81
& 73.04 & 1.47
& 65.42 & 63.00 & 77.69 & 78.54 & 94.70 & \textbf{75.87} \\
\rowcolor{fluyellow!25} Qwen-3B & SST
& 77.94 & 94.12 & 64.32
& \textbf{78.79} & \textbf{0.77}
& 65.47 & 63.61 & 77.15 & 78.13 & 94.50 & 75.77 \\
\midrule
Llama-3B & Base
& 20.99 & 18.43 & 49.59
& 29.67 & --
& 62.12 & 67.48 & 73.90 & 73.96 & 94.20 & 74.33 \\
Llama-3B & Seq-RL
& 71.93 & 91.76 & 59.81
& 74.50 & 1.47
& 61.84 & 67.32 & 73.23 & 73.71 & 94.30 & 74.08 \\
\rowcolor{fluyellow!25} Llama-3B & SST
& 77.94 & 92.55 & 62.81
& \textbf{77.77} & \textbf{1.27}
& 61.95 & 68.54 & 73.50 & 72.56 & 94.30 & \textbf{74.17} \\
\bottomrule
\end{tabular}
}
\vspace{-2mm}
\end{table*}

%% file: algo/sst.tex
\begin{wrapfigure}{r}{0.5\columnwidth}
\begin{minipage}{0.48\columnwidth}
\vspace{-4mm}
\begin{algorithm}[H]
\caption{\small Skill Schema Transport (SST)}
\label{alg:sst}
\footnotesize
\begin{algorithmic}[1]

\Require Pre-trained model $\mathcal{M}_{\theta_0}$
\Require Update stream $\mathcal{D}_1,\ldots,\mathcal{D}_T$
\Require Distance $\mathrm{Dist}$; thresholds $\epsilon$, $\delta_{\mathrm{create}}$, $\delta_{\mathrm{merge}}$, $\delta_{\mathrm{drop}}$, $n_{\min}$, $k_{\min}$
\Ensure Updated model $\mathcal{M}_{\theta_T}$, schema memory $\mathcal{B}_T$

\State Initialize $\mathcal{B}_0$ from seed skill families

\For{$t = 1, \ldots, T$}
    \State $\bar{\mathcal{B}} \leftarrow \mathcal{B}_{t-1}$; \quad $\widetilde{\mathcal{D}}_t \leftarrow \emptyset$
    \For{each $(x, y, \tau) \sim \mathcal{D}_t$}
        \State $z \leftarrow \Gamma_{\theta_{t-1}}(\tau)$
        \State $z^\star \leftarrow \arg\min_{z' \in \bar{\mathcal{B}}} \mathrm{Dist}(z, z')$
        \State $\widetilde z \leftarrow z^\star$ if $\mathrm{Dist}(z,z^\star) \le \epsilon$ else $z$
        \State Add $(x,\widetilde z,y)$ to $\widetilde{\mathcal{D}}_t$
        \State $\bar{\mathcal{B}} \leftarrow \mathrm{MemUpdate}_{\Gamma_{\theta_{t-1}}}(\bar{\mathcal{B}}, z, \widetilde z)$
    \EndFor
    \State $\mathcal{B}_t \leftarrow \bar{\mathcal{B}}$
    \State $\theta_t \leftarrow \arg\min_\theta\ \mathcal{L}_{\mathrm{SST}}^t(\theta;\widetilde{\mathcal{D}}_t)$
\EndFor

\end{algorithmic}
\end{algorithm}
\end{minipage}
\vspace{-4mm}
\end{wrapfigure}

%% file: exp.tex
\textbf{Models and datasets.}
We evaluate SST on \texttt{Llama-3.2-3B-Instruct} ~\citep{grattafiori2024llama} and \texttt{Qwen2.5-3B-Instruct}~\citep{qwen2.5}.
We consider two streaming settings: a mathematical reasoning stream~\citep{hendrycksmath2021} with four sequential domains (Algebra $\to$ Counting \& Probability $\to$ Geometry $\to$ Intermediate Algebra), and an agent-task stream with three sequential domains (AlfWorld $\to$ WebShop $\to$ search-augmented QA~\citep{shridhar2020alfworld,yao2022webshop,xia2026skillrl}).

\textbf{Training and inference.}
At each stream step, methods may only access the current domain plus any information explicitly permitted by their design.
SST applies the same base objective as the sequential baseline (SFT for MATH, RL for agent tasks), but rewrites updates through retrieved skill schemas.
All main evaluations are retrieval-free at inference time, so reported gains reflect internalized knowledge rather than test-time prompting.

\textbf{Evaluation.}
After each step we evaluate on all domains seen so far. We report average final performance and average max-drop forgetting. For a domain $d$ introduced at step $t_d$, let $A_t(d)$ denote its accuracy after update step $t$. We define
\[
    F(d)=\max_{t_d\le t\le T} A_t(d)-A_T(d).
\]
We average $F(d)$ over prior streamed domains, excluding the final domain because it is not followed by a later update.
For the MATH stream, we additionally report transfer to held-out domains (Number Theory, Prealgebra, Precalculus) and preservation on general benchmarks (MMLU~\citep{hendrycks2020measuring}, GSM8K~\citep{cobbe2021training}, PIQA~\citep{bisk2020piqa}, CommonsenseQA~\citep{talmor2019commonsenseqa}, SciQ~\citep{welbl2017crowdsourcing}); the same benchmarks are used for the agent stream.

\textbf{Baselines.}
We compare against \textsc{SeqTrain} (sequential, no mitigation), \textsc{Replay} (fixed-count rehearsal), and \textsc{Regularization} (L2 parameter anchoring to the base model).
Unlike these, SST does not anchor to past examples or states; instead, it conditions new updates on reusable procedural schemas.
We also include non-streaming experiments, Appendix~\ref{app:single_domain_gain_drop} shows that the gain-drop pattern also holds for independent single-domain updates.

\input{exps/main_results}
\input{exps/baseline}

\input{exps/underfit}

\input{exps/heatmap_case}
\input{exps/streaming}

%% file: exps/main_results.tex
\noindent\textbf{RQ1: Does SST reduce forgetting while preserving out-of-stream performance?}
Table~\ref{tab:main_streaming} summarizes results across both streaming settings and backbones.
In the MATH stream, Seq-SFT causes severe forgetting on Qwen3B (avg.\ max-drop $42.41$ pts, final streamed average $19.96\%$); SST reduces max-drop to $7.10$ pts and raises the streamed average to $25.12\%$, while recovering general benchmark performance from $58.92\%$ to $72.06\%$.
Llama3B follows the same direction: streamed average improves from $20.30\%$ to $22.83\%$ and general benchmarks from $69.33\%$ to $72.33\%$.
In the agent-task stream, SST improves final average task performance over Seq-RL by $5.75$ pts on Qwen3B ($73.04\%\to78.79\%$) and $3.27$ pts on Llama3B ($74.50\%\to77.77\%$), with forgetting reduced and general benchmark scores preserved.
Across both streams, schema-conditioned updates improve streaming adaptation without increasing interference with out-of-stream capabilities.
Appendix~\ref{app:rq1} provides the corresponding max-drop decomposition by domain.

%% file: exps/baseline.tex
\begin{figure*}[t]
    \centering
    \includegraphics[width=0.92\textwidth]{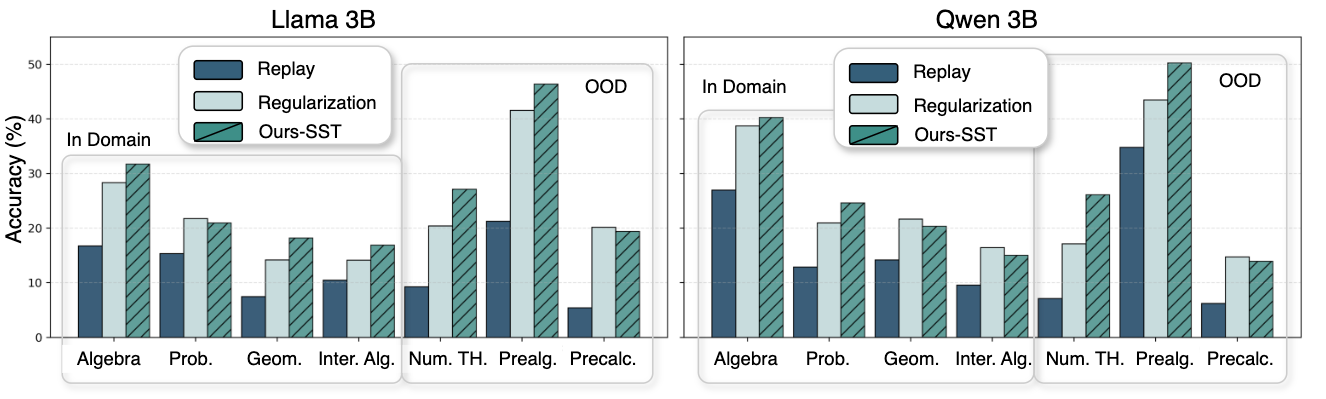}
    \caption{\small
    \textbf{Baseline comparison on streaming MATH.}
    Final accuracy after the four-domain stream. Left: seen domains. Right: held-out related MATH domains. SST yields a favorable retention-transfer trade-off on both models.
    }
    \label{fig:streaming_baselines}
    \vspace{-8mm}
\end{figure*}

\begin{wrapfigure}{r}{0.5\textwidth}
    \centering
    \vspace{-3mm}
    \includegraphics[width=0.49\textwidth]{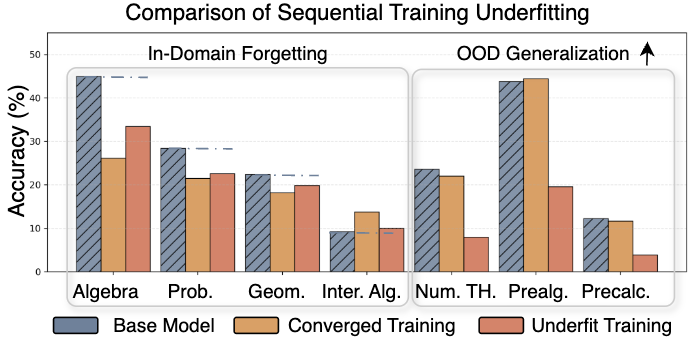}
    \caption{\small
    \textbf{Underfitting does not explain SST's retention.}
    Underfit SFT retains more earlier-domain accuracy only by failing to learn the final domain and degrading related OOD generalization.
    }
    \label{fig:underfit_control}
    \vspace{-2mm}
\end{wrapfigure}

\noindent\textbf{RQ2: How does SST compare with standard forgetting-mitigation baselines?}
Replay and regularization preserve prior behavior by anchoring updates to earlier examples or model states.
SST instead maps new examples into a reusable skill-schema space before updating, encouraging the model to relate new domains to previously acquired procedures.
Figure~\ref{fig:streaming_baselines} reports final accuracy after the four-domain MATH stream.
Across both backbones, SST matches or improves over replay and regularization on streamed domains, with the largest advantage on held-out related MATH domains, where preserving old outputs is insufficient and transferring reusable procedures is required.
This suggests schema transport does more than constrain model drift: by organizing updates around reusable skills, SST yields a favorable retention-transfer trade-off.

%% file: exps/underfit.tex
\noindent\textbf{RQ3: Can forgetting be avoided simply by underfitting the update?}
A simpler explanation for improved retention is weaker optimization: undertrained updates may overwrite earlier domains less.
To test this, we compare converged and underfit sequential SFT on Qwen-3B over the four-domain MATH stream.
As shown in Figure~\ref{fig:underfit_control}, underfit SFT preserves some earlier-domain accuracy but at the cost of Intermediate Algebra performance and held-out MATH generalization (Number Theory, Prealgebra, Precalculus), a retention, plasticity trade-off rather than a solution to continual adaptation.
SST targets a different mechanism: by conditioning updates on reusable skill schemas, it preserves prior procedures while still fitting new examples.

%% file: exps/heatmap_case.tex
\noindent\textbf{RQ4: Do later updates overwrite prior skills, or reinforce shared procedures?}
Final-checkpoint accuracy captures the net effect of the stream but not how skills evolve during adaptation.
We evaluate each SST checkpoint on all streamed MATH domains after every update, testing whether shared procedural structure allows later updates to recover or reinforce earlier skills.
For each prior domain $d$, max-drop forgetting is $F(d)=\max_{t_d \leq t \leq T} A_t(d)-A_T(d)$.

\begin{wrapfigure}{r}{0.65\textwidth}
    \centering
    \vspace{-4mm}
    \includegraphics[width=0.64\textwidth]{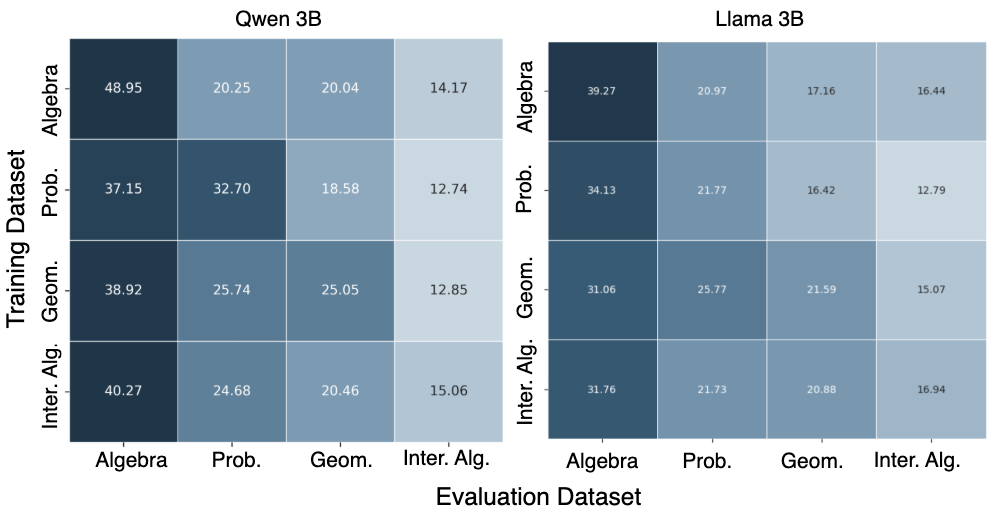}
    \caption{\small
    \textbf{SST checkpoint trajectories on streaming MATH.}
    Rows are checkpoints after each update and columns are evaluation domains.
    Partial recoveries after later updates are consistent with later domains reinforcing shared procedures.
    }
    \label{fig:streaming_trajectory_heatmap}
    \vspace{-2mm}
\end{wrapfigure}

Figure~\ref{fig:streaming_trajectory_heatmap} shows that SST keeps forgetting bounded: Qwen-3B drops by $8.68$, $8.02$, and $4.59$ points on Algebra, Counting \& Probability, and Geometry (avg.\ $7.10$ pts), Llama-3B drops by $7.51$, $4.04$, and $0.71$ points on Algebra, Counting \& Probability, and Geometry, respectively, with an average max-drop of $4.09$ points.
Crucially, trajectories are not purely monotonic: Qwen-3B's Algebra accuracy partially recovers from $37.15\%$ to $40.27\%$ by the final checkpoint, consistent with later examples activating schemas shared with earlier domains.
These trajectories are consistent with SST mitigating forgetting not by freezing prior behavior, but by allowing skills to be reused and sometimes reinforced as the stream grows.

%% file: exps/streaming.tex
\begin{figure}[b]
    \centering
    \vspace{-2mm}
    \includegraphics[width=0.65\textwidth]{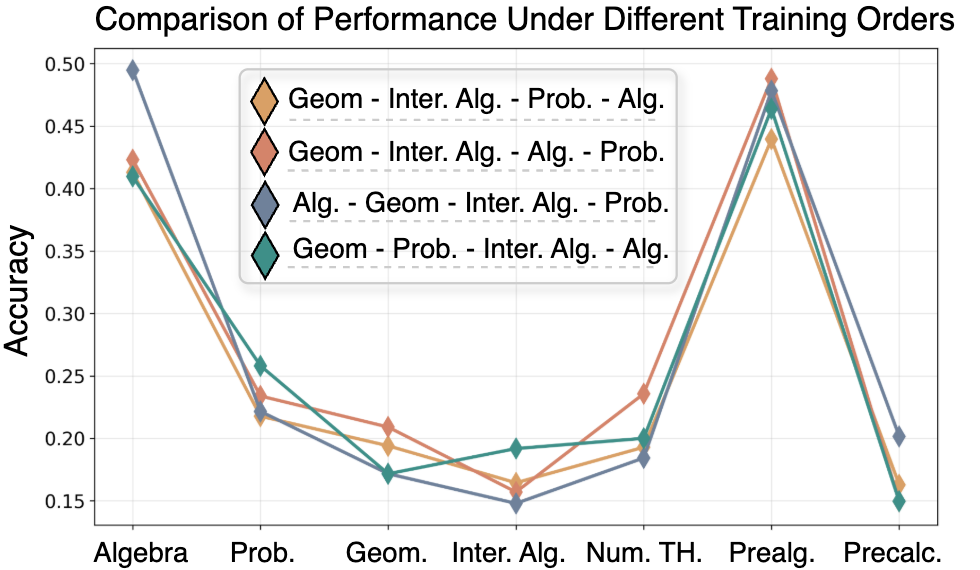}
    \caption{\small
    \textbf{Robustness to stream order.}
    SST is trained on four permutations of the same MATH stream and evaluated after the final update.
    Performance is broadly consistent across orders, indicating that SST does not rely on a single favorable domain sequence.
    }
    \label{fig:stream_order}
    \vspace{-3mm}
\end{figure}

\noindent\textbf{RQ5: Is SST robust to the order in which domains arrive?}
Streaming adaptation is inherently order-sensitive: a method may appear effective only because the domains are presented in a favorable sequence.
We therefore test SST on four permutations of the same MATH domains using the same training budget and evaluation protocol, and evaluate the final checkpoint on both streamed and held-out domains.
Figure~\ref{fig:stream_order} shows that SST yields a broadly consistent performance profile across orders.
Although some variation remains, especially on Algebra and Number Theory, no tested permutation causes a collapse on streamed domains, and held-out domains such as Prealgebra remain comparatively strong.
This finite-order stress test does not prove full order invariance, but it reduces the concern that SST's retention-transfer trade-off is driven by a single hand-picked curriculum.
Instead, schema transport appears to provide a stable procedural interface for integrating new domains as they arrive.

%% file: appendix.tex
\input{appen/add_rq}
\input{appen/proof}
\input{appen/extrarq1}

\input{appen/implement}

\input{appen/prompt}
\input{appen/limitation}
\input{appen/border_impact}

%% file: appen/add_rq.tex
\begin{figure}[H]
    \centering
    \vspace{-4mm}
    \includegraphics[width=0.64\textwidth]{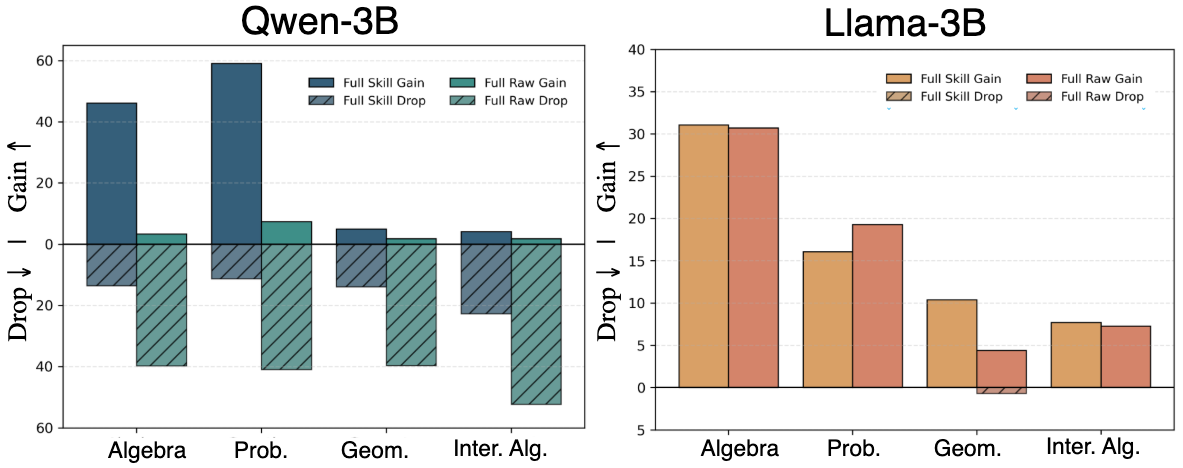}
    \caption{\small
\textbf{Single-domain gain--drop diagnostic.}
Each MATH domain is trained independently from the same base checkpoint.
Upward bars show target-domain gain, and downward hatched bars show average drop on the other streamed MATH domains.
Compared with \textsc{Full Raw}, \textsc{Full Skill} yields a stronger gain--drop trade-off on both Qwen-3B and Llama-3B.
}
    \label{fig:single_domain_gain_drop}
    \vspace{-3mm}
\end{figure}

\section{Single-Domain Gain-Drop Diagnostic}
\label{app:single_domain_gain_drop}

The main experiments evaluate SST in streaming settings, where updates arrive sequentially and later domains may interact with earlier ones.
To isolate the effect of schema conditioning from stream order and cumulative checkpoint dynamics, we additionally run a non-streaming single-domain diagnostic.
For each target domain $d \in \{\text{Algebra}, \text{Counting \& Probability}, \text{Geometry}, \text{Intermediate Algebra}\}$, we initialize from the same base model and train only on examples from $d$.
We compare raw full fine-tuning against skill-conditioned full fine-tuning using the same target-domain data and training budget.
All evaluations in this diagnostic are retrieval-free at inference time; skill schemas are used only as training-time context.

Let $A_0(e)$ denote the base-model accuracy on evaluation domain $e$, and let $A_d^{m}(e)$ denote the accuracy after training on target domain $d$ with method $m \in \{\textsc{Raw}, \textsc{Skill}\}$.
We measure the target-domain gain and cross-domain drop as
\[
\mathrm{Gain}_{m}(d) = A_d^{m}(d) - A_0(d),
\]
\[
\mathrm{Drop}_{m}(d)
=
\frac{1}{|\mathcal E|-1}
\sum_{e \in \mathcal E \setminus \{d\}}
\left[A_0(e) - A_d^{m}(e)\right]_+,
\]
where $\mathcal E$ is the set of streamed MATH evaluation domains and $[\cdot]_+$ counts only performance decreases.
In Figure~\ref{fig:single_domain_gain_drop}, gains are plotted upward and drops are plotted downward as $-\mathrm{Drop}$.

Figure~\ref{fig:single_domain_gain_drop} shows that schema conditioning improves the gain-drop trade-off even in isolated single-domain updates.
Starting from the same base checkpoint, raw fine-tuning often provides weaker target-domain gains and/or larger drops on the remaining streamed MATH domains.
In contrast, skill-conditioned fine-tuning achieves stronger or comparable in-domain gains with smaller cross-domain degradation.
This complements the streaming results by showing that SST's advantage is not only a consequence of cumulative sequential adaptation, but also appears at the level of individual domain updates.

%% file: appen/proof.tex
\section{Additional Details for the Linearized SST Analysis}
\label{app:sst_theory_details}

\subsection{A Linearized Drift Bound}
\label{subsec:linearized_bound}

We give a local linearized argument for why schema conditioning can reduce retention drift. The result concerns minimum shared-parameter displacement in a first-order logit-fitting problem, not global convergence. Assume the model decomposes into shared parameters $\theta_{\mathrm{sh}}$
and schema-specific parameters $\theta_{\mathrm{sc}}$, with logits
\begin{equation}
    g_\theta(x,s)
    =
    f\!\left(W\phi_{\theta_{\mathrm{sh}}}(x)+Ue_{\theta_{\mathrm{sc}}}(s)\right).
\end{equation}
Under the following assumption, retention diagnostics evaluated with $s_0$
depend only on $\theta_{\mathrm{sh}}$ to first order.

\begin{assumption}[Null-schema inactivity]
\label{ass:null_schema}
The schema embedding satisfies $e_{\theta_{\mathrm{sc}}}(s_0)\equiv 0$
for all $\theta_{\mathrm{sc}}$, so that
\begin{equation}
    \left.
    \frac{\partial g_\theta(x,s_0)}
    {\partial \theta_{\mathrm{sc}}}
    \right|_{\bar\theta}
    = 0.
\end{equation}
\end{assumption}

At step $t$, let $\bar\theta=\theta_{t-1}$
and let $r\in\mathbb{R}^m$ be the desired first-order logit change for a
mini-batch from $\mathcal{D}_t$. The linearized updates are
\begin{equation}
\label{eq:linearized_changes}
    \Delta^{\mathrm{SFT}} = A\delta_{\mathrm{sh}},
    \qquad
    \Delta^{\mathrm{SST}} = A\delta_{\mathrm{sh}}+S\delta_{\mathrm{sc}},
\end{equation}
where $A$, $S$ are Jacobians w.r.t.\ shared and schema-specific parameters.
Let $P_A=AA^\dagger$, $r_A=P_Ar$, $\widetilde{S}=P_AS$, so that $r_A$ is
the component of $r$ fitable by shared parameters and $\widetilde{S}$
captures how schema-specific directions can offload residual from
$\theta_{\mathrm{sh}}$.

Define the minimum shared-parameter displacement under each objective as
\begin{equation}
    \rho_{\mathrm{SFT}}^2
    =
    \min_{\delta_{\mathrm{sh}}}
    \left\{\|\delta_{\mathrm{sh}}\|_2^2: A\delta_{\mathrm{sh}}=r_A\right\},
    \qquad
    \rho_{\mathrm{SST}}^2
    =
    \min_{\delta_{\mathrm{sh}},\delta_{\mathrm{sc}}}
    \left\{\|\delta_{\mathrm{sh}}\|_2^2:
    A\delta_{\mathrm{sh}}+\widetilde{S}\delta_{\mathrm{sc}}=r_A\right\}.
\end{equation}
Since retention diagnostics under $s_0$ depend only on $\theta_{\mathrm{sh}}$,
$\delta_{\mathrm{sc}}$ is free in $\rho_{\mathrm{SST}}^2$.

\begin{theorem}[SST cannot increase the minimum required shared-parameter change]
\label{thm:sst_shared_norm}
Let $M=(AA^\top)^\dagger$ and
$\widetilde{\mathcal{C}}=\operatorname{col}(\widetilde{S})
\subseteq\operatorname{col}(A)$. Then
\begin{equation}
\label{eq:rho_sst_formula}
    \rho_{\mathrm{SST}}^2
    =
    r_A^\top M r_A
    -
    r_A^\top M \widetilde{S}
    \left(\widetilde{S}^\top M \widetilde{S}\right)^\dagger
    \widetilde{S}^\top M r_A
    \leq
    \rho_{\mathrm{SFT}}^2,
\end{equation}
where $\|u\|_M^2=u^\top M u$ for $u\in\operatorname{col}(A)$. The inequality is strict if and only if
$\widetilde{S}^\top M r_A \neq 0$.
\end{theorem}

\begin{proof}
If $A=0$, then $r_A=0$ and the claim is trivial. Otherwise, let
$A=U\Sigma V^\top$ be a compact SVD. Then
$M=(AA^\top)^\dagger=U\Sigma^{-2}U^\top$. Hence
$\langle u,v\rangle_M=u^\top Mv$ is an inner product on
$\operatorname{col}(A)$: for every nonzero $u=Uw\in\operatorname{col}(A)$,
we have $u^\top Mu=w^\top\Sigma^{-2}w>0$.

For any $b\in\operatorname{col}(A)$, the minimum-norm solution of
$A\delta_{\mathrm{sh}}=b$ is $A^\dagger b$. Therefore
\begin{equation}
    \min_{\delta_{\mathrm{sh}}}
    \left\{
        \|\delta_{\mathrm{sh}}\|_2^2:
        A\delta_{\mathrm{sh}}=b
    \right\}
    =
    \|A^\dagger b\|_2^2
    =
    b^\top (A^\dagger)^\top A^\dagger b
    =
    b^\top (AA^\top)^\dagger b
    =
    \|b\|_M^2.
\end{equation}
The identity $(A^\dagger)^\top A^\dagger=(AA^\top)^\dagger$ follows directly from the same SVD. Taking $b=r_A\in\operatorname{col}(A)$ gives
$\rho_{\mathrm{SFT}}^2=\|r_A\|_M^2$.

For SST, write $s=\widetilde{S}\delta_{\mathrm{sc}}\in\widetilde{\mathcal{C}}$. For any fixed $s$, the remaining residual assigned to shared parameters is $r_A-s\in\operatorname{col}(A)$, and the minimum shared-parameter norm is $\|r_A-s\|_M^2$ by the previous paragraph. Minimizing over $s\in\widetilde{\mathcal{C}}$ gives
\begin{equation}
    \rho_{\mathrm{SST}}^2
    =
    \min_{s\in\widetilde{\mathcal{C}}}\|r_A-s\|_M^2.
\end{equation}
Thus $\rho_{\mathrm{SST}}^2$ is the squared $M$-distance from $r_A$ to
$\widetilde{\mathcal{C}}$. The $M$-orthogonal projection onto
$\widetilde{\mathcal{C}}$ is
\begin{equation}
    P_{\widetilde{\mathcal{C}}}^{M}r_A
    =
    \widetilde{S}
    \left(\widetilde{S}^\top M \widetilde{S}\right)^\dagger
    \widetilde{S}^\top M r_A.
\end{equation}
Therefore,
\begin{align}
    \rho_{\mathrm{SST}}^2
    &=
    \left\|r_A-P_{\widetilde{\mathcal{C}}}^{M}r_A\right\|_M^2 \\
    &=
    \|r_A\|_M^2
    -
    \left\|P_{\widetilde{\mathcal{C}}}^{M}r_A\right\|_M^2 \\
    &=
    r_A^\top M r_A
    -
    r_A^\top M \widetilde{S}
    \left(\widetilde{S}^\top M \widetilde{S}\right)^\dagger
    \widetilde{S}^\top M r_A.
\end{align}
This proves~\eqref{eq:rho_sst_formula}. Since $0\in\widetilde{\mathcal{C}}$,
\begin{equation}
    \rho_{\mathrm{SST}}^2
    =
    \min_{s\in\widetilde{\mathcal{C}}}\|r_A-s\|_M^2
    \leq
    \|r_A\|_M^2
    =
    \rho_{\mathrm{SFT}}^2.
\end{equation}
The inequality is strict exactly when
$P_{\widetilde{\mathcal{C}}}^{M}r_A\neq0$, which is equivalent to
$r_A$ not being $M$-orthogonal to $\widetilde{\mathcal{C}}$, i.e.,
$\widetilde{S}^\top M r_A\neq0$.
\end{proof}

\begin{corollary}[Incremental retention-drift bound]
\label{cor:incremental_drift_bound}
Let $G_{\bar\theta}(\theta;\mathcal{Q})$ be a nonnegative incremental retention
diagnostic evaluated under $s_0$ that depends locally only on
$\theta_{\mathrm{sh}}$, satisfies $G_{\bar\theta}(\bar\theta;\mathcal{Q})=0$,
and is $K_{\mathcal{Q}}$-Lipschitz in $\theta_{\mathrm{sh}}$ near $\bar\theta$.
For the minimum-shared-norm SST update,
\begin{equation}
    G_{\bar\theta}(\bar\theta+\delta^{\mathrm{SST}};\mathcal{Q})
    \leq K_{\mathcal{Q}}\rho_{\mathrm{SST}}
    \leq K_{\mathcal{Q}}\rho_{\mathrm{SFT}}.
\end{equation}
For a raw-reference diagnostic $G^0$, this bounds only the increment
$|G^0(\theta;\mathcal{Q})-G^0(\bar\theta;\mathcal{Q})|$, not
$G^0(\theta;\mathcal{Q})$ itself.
\end{corollary}
\begin{proof}
See Appendix~\ref{app:proof_corollary_drift}.
\end{proof}

\begin{corollary}[Retention-loss-increase bound]
\label{cor:forgetting_gap_bound}
Assume the retention loss $\mathcal{L}_{\mathcal{Q}}$ evaluated under $s_0$
depends locally only on $\theta_{\mathrm{sh}}$ and is $K_{\mathcal{Q}}$-Lipschitz
in $\theta_{\mathrm{sh}}$ near $\bar\theta$. Define
\begin{equation}
    \Delta\mathcal{L}_{\mathcal{Q}}(\delta)
    =
    \mathcal{L}_{\mathcal{Q}}(\bar\theta+\delta)
    -
    \mathcal{L}_{\mathcal{Q}}(\bar\theta).
\end{equation}
Then, for any minimum-shared-norm SST update,
\begin{equation}
    \left[\Delta\mathcal{L}_{\mathcal{Q}}(\delta^{\mathrm{SST}})\right]_+
    \leq K_{\mathcal{Q}}\rho_{\mathrm{SST}}
    \leq K_{\mathcal{Q}}\rho_{\mathrm{SFT}}.
\end{equation}
\end{corollary}
\begin{proof}
See Appendix~\ref{app:proof_corollary_forgetting}.
\end{proof}

\subsection{Proof of Corollary~\ref{cor:incremental_drift_bound}}
\label{app:proof_corollary_drift}

By Assumption~\ref{ass:null_schema} and the local schema-separable
decomposition, $G_{\bar\theta}(\cdot\,;\mathcal{Q})$ depends locally only
on $\delta_{\mathrm{sh}}$. Since $G_{\bar\theta}(\bar\theta;\mathcal{Q})=0$,
nonnegativity and Lipschitz continuity give
\begin{equation}
    G_{\bar\theta}(\bar\theta+\delta;\mathcal{Q})
    =
    \left|
    G_{\bar\theta}(\bar\theta+\delta;\mathcal{Q})
    -
    G_{\bar\theta}(\bar\theta;\mathcal{Q})
    \right|
    \leq
    K_{\mathcal{Q}}\|\delta_{\mathrm{sh}}\|_2.
\end{equation}
For the minimum-shared-norm SST update,
$\|\delta_{\mathrm{sh}}^{\mathrm{SST}}\|_2=\rho_{\mathrm{SST}}$.
Theorem~\ref{thm:sst_shared_norm} gives $\rho_{\mathrm{SST}}\leq\rho_{\mathrm{SFT}}$.

\subsection{Proof of Corollary~\ref{cor:forgetting_gap_bound}}
\label{app:proof_corollary_forgetting}

By Assumption~\ref{ass:null_schema}, $\mathcal{L}_{\mathcal{Q}}$ evaluated
under $s_0$ depends locally only on $\theta_{\mathrm{sh}}$ and is
$K_{\mathcal{Q}}$-Lipschitz. Therefore
\begin{equation}
    \left|
    \mathcal{L}_{\mathcal{Q}}(\bar\theta+\delta)
    -
    \mathcal{L}_{\mathcal{Q}}(\bar\theta)
    \right|
    \leq
    K_{\mathcal{Q}}\|\delta_{\mathrm{sh}}\|_2,
\end{equation}
which gives
\begin{equation}
    \left[\Delta\mathcal{L}_{\mathcal{Q}}(\delta)\right]_+
    \leq
    K_{\mathcal{Q}}\|\delta_{\mathrm{sh}}\|_2.
\end{equation}
For the minimum-shared-norm SST update,
$\|\delta_{\mathrm{sh}}^{\mathrm{SST}}\|_2=\rho_{\mathrm{SST}}$.
Theorem~\ref{thm:sst_shared_norm} gives $\rho_{\mathrm{SST}}\leq\rho_{\mathrm{SFT}}$.

\subsection{Exact interpolation versus projected fitable component}
\label{app:exact_vs_projected}

The main theorem compares the shared-parameter displacement needed for the component $r_A=P_A r$ that is fitable by the shared Jacobian. This is the appropriate comparison when the linearized system is interpreted in least-squares or projected-logit form. If instead one requires exact interpolation of the full residual $r$, the feasible sets should be written explicitly.

For exact SFT interpolation,
\begin{equation}
    \rho_{\mathrm{SFT,exact}}^2
    =
    \min_{\delta_{\mathrm{sh}}}
    \left\{
        \|\delta_{\mathrm{sh}}\|_2^2:
        A\delta_{\mathrm{sh}}=r
    \right\}
\end{equation}
is finite if and only if $r\in\operatorname{col}(A)$. For exact SST interpolation,
\begin{equation}
    \rho_{\mathrm{SST,exact}}^2
    =
    \min_{\delta_{\mathrm{sh}},\delta_{\mathrm{sc}}}
    \left\{
        \|\delta_{\mathrm{sh}}\|_2^2:
        A\delta_{\mathrm{sh}}+S\delta_{\mathrm{sc}}=r
    \right\}
\end{equation}
is finite if and only if $r\in\operatorname{col}(A)+\operatorname{col}(S)$. Whenever this exact SST problem is feasible, it can be written as
\begin{equation}
    \rho_{\mathrm{SST,exact}}^2
    =
    \min_{s\in\operatorname{col}(S):\; r-s\in\operatorname{col}(A)}
    \|r-s\|_M^2.
\end{equation}
In the special case where $r\in\operatorname{col}(A)$, exact SFT is feasible and the exact SST expression reduces to
\begin{equation}
    \rho_{\mathrm{SST,exact}}^2
    =
    \min_{s\in\operatorname{col}(S)\cap\operatorname{col}(A)}
    \|r-s\|_M^2
    \leq
    \|r\|_M^2
    =
    \rho_{\mathrm{SFT,exact}}^2.
\end{equation}
Thus exact interpolation also cannot increase the minimum required shared-parameter displacement. The projected theorem in the main text is slightly more permissive because it studies the shared-fitable component $r_A$ and the projected schema directions $P_A S$; the exact full-residual theorem instead uses the intersection $\operatorname{col}(S)\cap\operatorname{col}(A)$ when $r\in\operatorname{col}(A)$.

\subsection{Counterexample to the ordinary Euclidean-projection claim}
\label{app:counterexample_euclidean_projection}

The following example shows why the ordinary Euclidean projection $P_S$ onto $\operatorname{col}(S)$ is not sufficient for the theorem.

Let
\begin{equation}
    M=
    \begin{pmatrix}
    1 & -1\\
    -1 & 2
    \end{pmatrix},
    \qquad
    M^{-1}=
    \begin{pmatrix}
    2 & 1\\
    1 & 1
    \end{pmatrix},
\end{equation}
and choose a full-row-rank matrix $A$ such that $AA^\top=M^{-1}$, for example
\begin{equation}
    A=
    \begin{pmatrix}
    \sqrt{2} & 0\\
    1/\sqrt{2} & 1/\sqrt{2}
    \end{pmatrix}.
\end{equation}
Let
\begin{equation}
    S=
    \begin{pmatrix}
    1\\
    0
    \end{pmatrix},
    \qquad
    r=
    \begin{pmatrix}
    1\\
    1
    \end{pmatrix}.
\end{equation}
The ordinary Euclidean projection of $r$ onto $\operatorname{col}(S)$ is
\begin{equation}
    P_Sr=
    \begin{pmatrix}
    1\\
    0
    \end{pmatrix}
    \neq0.
\end{equation}
However,
\begin{equation}
    S^\top Mr
    =
    \begin{pmatrix}1&0\end{pmatrix}
    \begin{pmatrix}
    1 & -1\\
    -1 & 2
    \end{pmatrix}
    \begin{pmatrix}
    1\\
    1
    \end{pmatrix}
    =0.
\end{equation}
Therefore $P_{\operatorname{col}(S)}^M r=0$ and there is no strict improvement:
\begin{equation}
    \rho_{\mathrm{SFT}}^2
    =
    r^\top Mr
    =1,
    \qquad
    \rho_{\mathrm{SST}}^2
    =
    \min_{a\in\mathbb{R}}
    \begin{pmatrix}1-a&1\end{pmatrix}
    M
    \begin{pmatrix}1-a\\1\end{pmatrix}
    =
    \min_{a\in\mathbb{R}}(1+a^2)
    =1.
\end{equation}
Moreover, fitting the ordinary Euclidean residual $(I-P_S)r=(0,1)^\top$ through $A$ would require squared shared norm
\begin{equation}
    (0,1)M(0,1)^\top=2,
\end{equation}
which is larger than the SFT cost $1$. Hence the statement
\begin{equation}
    \rho_{\mathrm{SST}}^2
    =
    \min_{\delta_{\mathrm{sh}}}
    \left\{
        \|\delta_{\mathrm{sh}}\|_2^2:
        A\delta_{\mathrm{sh}}=(I-P_S)r
    \right\}
\end{equation}
is false in general. It becomes valid if $AA^\top$ is proportional to the identity, or if $P_S$ is replaced by the $M$-orthogonal projection.

\subsection{What the theorem does and does not prove}
\label{app:scope_of_theorem}

The theorem proves a statement about a local first-order interpolation problem: after linearizing the logits around the pre-update parameter $\bar\theta$, adding schema-specific tangent directions cannot increase the minimum amount of movement required in shared parameters to fit the same shared-fitable residual component. The theorem does not, by itself, prove that arbitrary nonconvex SGD will find the minimum-shared-norm solution, nor that every empirical forgetting metric, including accuracy max-drop, must decrease without the stated null-schema and Lipschitz assumptions. Theory-related claims should therefore be phrased as local upper bounds under explicit assumptions, while the actual retention behavior should be validated empirically.

%% file: appen/extrarq1.tex
\section{Additional Details for RQ1}
\label{app:rq1}

\paragraph{Forgetting metric.}
For each streamed domain $d$, let $t_d$ be the update step at which $d$ is introduced and let $T$ be the final update step.
We compute max-drop forgetting as
\begin{equation}
    F(d) = \max_{t_d \leq t \leq T} A_t(d) - A_T(d),
\end{equation}
where $A_t(d)$ is the evaluation accuracy on domain $d$ after update step $t$.
We average $F(d)$ over prior streamed domains only.
The final domain is excluded because it is not followed by a later update and therefore cannot be overwritten within the observed stream.

%% file: appen/implement.tex
\begin{table}[t]
\centering
\small
\caption{
\textbf{Default training and evaluation settings.}
}
\label{tab:hparams}
\begin{tabular}{llc}
\toprule
Category & Setting & Value \\
\midrule
\multicolumn{3}{l}{\textit{SFT hyperparameters}} \\
\midrule
SFT & Training framework & LlamaFactory \\
SFT & Maximum sequence length & 4096 \\
SFT & Precision & bf16 \\
SFT & Scheduler & cosine \\
SFT & Warmup ratio & 0.03 \\
SFT & Per-device train batch size & 2 \\
SFT & Gradient accumulation steps & 8 \\
SFT & Effective batch size & 16 \\
SFT & Training epochs & 5 \\
SFT & Full fine-tuning learning rate & $1\times 10^{-5}$ \\
SFT & Retrieved skills $k$ & 2 \\
\midrule
\multicolumn{3}{l}{\textit{Evaluation settings}} \\
\midrule
Eval & Decoding & deterministic \\
Eval & Temperature & 0.0 \\
Eval & Top-$p$ & 1.0 \\
Eval & Maximum new tokens & 512 \\
Eval & Inference engine & vLLM \\
Eval & Tensor parallel size & 1 \\
Eval & GPU memory utilization & 0.8 \\
Eval & Maximum number of sequences & 64 \\
\bottomrule
\end{tabular}
\vspace{-4mm}
\end{table}

\section{Implementation Details}
\label{app:implementation_details}

\paragraph{Sequential MATH adaptation.}
For the sequential skill-conditioned SFT setting, we train on four MATH domains in order: Algebra, Counting \& Probability, Geometry, and Intermediate Algebra. After each stage, we evaluate the resulting checkpoint on all seven Hendrycks MATH categories: Algebra, Counting \& Probability, Geometry, Intermediate Algebra, Number Theory, Prealgebra, and Precalculus. For SST, we construct skill-augmented training examples by retrieving the top-$k$ skills from the schema memory, with $k=2$ by default, and use the full skill prompt format described in Appendix~\ref{app:prompts}. Non-SST baselines use the same training examples and hyperparameters, but without the retrieved skill context.

\paragraph{SFT training configuration.}
We implement SFT using LlamaFactory~\citep{zheng2024llamafactory} and run evaluation with vLLM~\citep{kwon2023efficient}. Unless otherwise specified, we use bf16 training, cosine learning-rate scheduling, and a warmup ratio of $0.03$. The default per-device batch size is $2$ with gradient accumulation over $8$ steps, giving an effective batch size of $16$ on a single GPU. We train for $5$ epochs by default. For full fine-tuning, we use a learning rate of $1\times 10^{-5}$;

\paragraph{RL implementation details.}
For the agent-stream experiments, we instantiate SST-RL as a schema-conditioned offline reinforcement-learning setup.
In each agent domain, SST-RL augments the state prompt with the top retrieved skills from the schema memory and then optimizes the policy with a GRPO-style objective.
Seq-RL and SST-RL use the same RL objective, reward functions, decoding settings, and training budget; the only difference is that SST-RL conditions the training prompt on retrieved skill schemas.
In our implementation, the agent skill prompt uses top-$k$ retrieval with $k=4$.

Rewards are rule-based and combine action correctness with output-format compliance.
The main reward is an exact-action reward: it assigns $1.0$ when the generated action exactly matches the reference action and $0.0$ otherwise.
An auxiliary format reward adds $0.2$ when the response follows the required tagged protocol, including the expected \texttt{<think>} and \texttt{<action>} fields.
Unless otherwise specified, GRPO training uses learning rate $10^{-6}$, one training epoch, per-device batch size $1$, gradient accumulation $4$, four sampled completions per prompt, maximum prompt length $4096$, maximum completion length $512$, temperature $0.9$, $\beta=0.0$, and a DAPO-style loss.

\paragraph{Agent evaluation.}
Agent evaluation is action-level rather than full episodic environment evaluation.
The main metric reported for AlfWorld, WebShop, and Search is exact action-match accuracy under deterministic decoding.
For Search, we additionally compute tool-selection accuracy, average query-token F1, and a soft search accuracy that requires both the correct tool and query-token F1 of at least $0.5$.
All agent results in Table~\ref{tab:main_streaming} use retrieval-free inference unless explicitly stated otherwise.
Thus, the agent-stream experiment tests whether schema-conditioned training improves sequential robustness in agent decision-making, rather than introducing a new RL objective or relying on test-time skill retrieval.

\paragraph{Evaluation configuration.}
We use deterministic decoding for all evaluations, with temperature $0.0$ and top-$p=1.0$, and generate up to $512$ new tokens per example. For vLLM inference, we use tensor parallel size $1$. Typical inference settings use \texttt{gpu\_memory\_utilization=0.8} and \texttt{max\_num\_seqs=64}. All main results reported in Table~\ref{tab:main_streaming} use retrieval-free inference: the final fine-tuned model is evaluated with the raw solver prompt and without prepended skill blocks. Retrieval-enabled inference is used only as an optional diagnostic ablation, and any such results are reported separately when included.

\paragraph{Compute.}
All main SFT jobs are run on a single-GPU node, typically using $1\times$A100 GPU, $8$ CPU cores, and $256$GB host memory, with a $72$-hour wall-clock limit for training jobs. Base-model evaluation jobs use the same single-GPU setup with a shorter time budget. In our local setup, the full Hendrycks MATH test set contains $5{,}000$ examples across seven subject categories.

\paragraph{Replay baseline.}
Our replay baseline uses fixed-count rehearsal rather than a dynamic reservoir buffer.
At stream step $t$, the training set contains all current-domain examples together with a uniformly sampled subset from each previous domain.
Unless otherwise specified, we sample 500 examples per previous domain using a fixed random seed and sample without replacement after shuffling.
Replay examples are concatenated with current-domain examples without additional weighting.
The replay ratio is therefore step-dependent:
\[
\mathrm{replay\_ratio}_t = \frac{(t-1)\times 500}{|\mathcal D_t|}.
\]

\paragraph{Regularization baseline.}
The regularization baseline is an L2-SP-style parameter anchoring baseline.
It loads the original base model as an anchor and adds
\[
\lambda \sum_{\theta_i \in \Theta_{\mathrm{train}}} \|\theta_i-\theta_{0,i}\|_2^2
\]
to the sequential SFT objective over trainable parameters.
Unless otherwise specified, we use $\lambda=10^{-4}$ and the same optimizer, sequence length, batch size, learning-rate schedule, and stream order as the sequential SFT baseline.

\paragraph{Error-clustering diagnostic.}
We construct Figure~\ref{fig:error_clustering} by first evaluating the base model on MATH examples and retaining examples answered incorrectly under the same answer-normalization protocol used in Section~5.
Each failed example is represented by the \emph{problem text only} (i.e., excluding both the gold solution and the model-generated reasoning trace), embedded with the sentence embedding model \texttt{BAAI/bge-small-en-v1.5}, and projected to two dimensions with PCA using the top two principal components.
Procedural error labels are assigned by \emph{manual annotation} using the categories \textsc{algebraic manipulation}, \textsc{case omission / counting breakdown}, \textsc{geometric relation misuse}, \textsc{modular / divisibility reasoning}, \textsc{coordinate-modeling error}, \textsc{constraint / interval checking}, and \textsc{answer-format / finalization}.
Domain labels are the original MATH subject categories.
To avoid circularity, the procedural labels are used only for post-hoc visualization and are not used as supervision, retrieval targets, memory entries, or training signals for SST.
This diagnostic is intended only as motivating evidence that failure modes exhibit cross-domain procedural overlap, rather than as a quantitative claim about a complete latent error taxonomy.

%% file: appen/prompt.tex
\section{Prompting Details}
\label{app:prompts}

\paragraph{Overview.}
We use three prompt types in the SST pipeline: a raw solver prompt for base and main retrieval-free evaluation, a skill-conditioned solver prompt for SST training, and a memory-update prompt for maintaining the skill memory. Unless explicitly stated otherwise, all reported main results use the raw solver prompt at inference time without retrieved skill blocks. In all skill-conditioned SFT examples, retrieved skills are provided only as input context; the supervision target remains the gold solution.

\begin{tcolorbox}[fontupper=\small,title=\textbf{Raw solver prompt}]
{\ttfamily
Solve the following competition math problem. Show concise reasoning and final answer.\\
Type: \{problem\_type\}\\
Level: \{level\}\\[1mm]
Problem:\\
\{problem\}
}
\end{tcolorbox}

\paragraph{Skill-conditioned training prompt.}
For SST, we prepend the top-$k$ retrieved skill blocks to the problem input, with $k=2$ by default. The training target is the gold solution only; the model is not trained to reproduce the retrieved skills. If no skill is retrieved, the prompt includes a short fallback line indicating that the model should solve directly.

\begin{tcolorbox}[fontupper=\small,title=\textbf{Skill-conditioned solver prompt}]
{\ttfamily
Solve the following competition math problem. Use the retrieved skills when they fit the problem. Show concise reasoning and final answer.\\[1mm]
\{retrieved\_skill\_blocks\}\\[1mm]
Type: \{problem\_type\}\\
Level: \{level\}\\[1mm]
Problem:\\
\{problem\}
}
\end{tcolorbox}

\begin{tcolorbox}[fontupper=\small,title=\textbf{Fallback when no skill is retrieved}]
{\ttfamily
Relevant skills:\\
- None retrieved. Solve directly.
}
\end{tcolorbox}

\paragraph{Retrieved skill block format.}
Each retrieved skill is represented as a compact procedural description. The default prompt uses the full skill format below.

\begin{tcolorbox}[fontupper=\small,title=\textbf{Retrieved skill block}]
{\ttfamily
Relevant skill \{rank\}:\\
- Skill ID: \{skill\_id\}\\
- Family: \{skill\_family\}\\
- Goal: \{goal\}\\
- When to use: \{applicability\}\\
- Steps: \{canonical\_steps\}\\
- Common failure modes: \{failure\_modes\}
}
\end{tcolorbox}

\paragraph{Memory-update prompt.}
The memory-update prompt is used during bootstrapping and online memory maintenance. Given the current problem, the model output, and the retrieved skills, the model selects one memory operation from \textsc{reuse}, \textsc{revise}, \textsc{create}, \textsc{merge}, and \textsc{drop}. We require the output to be a single JSON object with no markdown or extra text. Reuse and drop target an existing skill; revise, create, and merge additionally provide a proposed skill entry.

\begin{tcolorbox}[fontupper=\small,title=\textbf{Memory-update prompt, condensed}]
{\ttfamily
You are maintaining a model-relative skill memory for competition math.\\
Given the current problem, the model's reasoning and final answer, and the retrieved skills, decide whether to reuse, revise, create, merge, or drop a skill.\\[1mm]
Rules:\\
- Prefer reuse over create.\\
- Prefer revise over create if the reasoning pattern is close.\\
- Create a new skill only when no existing skill covers the pattern.\\
- Use only the provided skill family IDs.\\
- Reuse, revise, merge, and drop must name exact existing skill IDs.\\
- Do not invent missing families, target IDs, or merge sources.\\
- Return exactly one JSON object with no markdown fences and no extra text.\\[1mm]
Return fields:\\
\{action, target\_skill\_id, proposed\_skill, reason, confidence\}\\[1mm]
Families:\\
\{families\_json\}\\[1mm]
Problem:\\
\{problem\}\\[1mm]
Model output:\\
\{prediction\}\\[1mm]
Retrieved skills:\\
\{retrieved\_skills\_json\}
}
\end{tcolorbox}

\paragraph{Chat template.}
Before vLLM inference, each text prompt is wrapped as a single user message and passed through the model tokenizer's chat template with generation enabled. This keeps the prompting format consistent across base-model evaluation, non-retrieval evaluation, and retrieval-conditioned evaluation.

%% file: appen/limitation.tex
\section{Limitations}
\label{app:limitations}

First, our empirical scope covers only two 3B backbones and two streaming settings; whether the same retention--transfer behavior holds for larger models, longer streams, multilingual or multimodal settings, or domains with weak prior procedural representation remains open.

Second, SST depends on schema extraction and retrieval quality. Incomplete, generic, or mismatched schemas provide weak or misleading conditioning, and inaccurate memory updates may fragment related skills or merge distinct ones. More systematic stress tests on noisy, ambiguous, or adversarial inputs are needed.

Third, the theoretical analysis is a local explanatory argument, not a global guarantee. The linearized result shows schema-conditioned directions need not increase shared-parameter displacement, but this does not guarantee nonconvex optimization finds such solutions in practice.

Fourth, we do not report full multi-seed statistics with confidence intervals for all comparisons. Future work should include repeated runs and larger-scale ablations over retrieval thresholds, memory-update rules, and prompt formats.

Fifth, SST introduces additional overhead for schema extraction, retrieval, and memory updates. Scaling behavior for very large streams or high-throughput settings remains unmeasured.

Finally, SST reduces but does not eliminate forgetting, and should be viewed as improving the retention--plasticity trade-off rather than solving continual learning.

%% file: appen/border_impact.tex
\section{Broader Impact}
\label{app:broader-impact}

SST may benefit applications requiring stable sequential model updates by reducing unnecessary overwriting of prior capabilities without retaining large amounts of old data. The explicit schema representation also offers a procedural interface for inspecting which reasoning patterns are reused or revised during training, which may assist in diagnosing failures and understanding how updates affect model behavior.

At the same time, more effective continual adaptation could amplify both beneficial and harmful uses, including harmful automation or unsafe agent behavior if deployed without safeguards. Learned schemas may also preserve and transfer biases or unsafe procedures encoded in training data, and should be evaluated not only for accuracy and forgetting but also for fairness, robustness, and safety in the intended deployment context.

Users should not over-interpret the explicit schema structure as a formal guarantee of model behavior: schemas are conditioning context for training, not certificates of internal reasoning. Any deployment should include independent evaluation, red-teaming, and monitoring for regressions. Future releases of code or schema memories should document data sources, licenses, intended use, and known limitations.